%% file: CHL.tex
\documentclass{article}

\usepackage{arxiv}

\usepackage[utf8]{inputenc} 
\usepackage[T1]{fontenc}    
\usepackage{hyperref}       
\usepackage{url}            
\usepackage{booktabs}       
\usepackage{amsfonts}       
\usepackage{nicefrac}       
\usepackage{microtype}      
\usepackage{lipsum}
\usepackage{amsfonts}       
\usepackage{amsmath}
\usepackage{amssymb}
\usepackage{tocbibind}
\usepackage[english]{babel}
\usepackage{amsthm}
\usepackage{pgf}
\usepackage[toc,page]{appendix}
\usepackage{wrapfig}
\usepackage{float}
\usepackage{xifthen}
\usepackage{environ}
\usepackage{tikz}
\usepackage{subcaption}
\usepackage{graphicx}

\newcommand{\HL}{\text{HL}}
\newcommand{\Exp}{\mathbb{E}}
\newcommand{\bs}{\mathcal{S}}
\newcommand{\X}{\mathcal{X}}

\newcommand\inputpgf[2]{{
\let\pgfimageWithoutPath\pgfimage
\renewcommand{\pgfimage}[2][]{\pgfimageWithoutPath[##1]{#1/##2}}
\input{#1/#2}
}}

\newtheorem{prop}{Proposition}

\pgfmathsetmacro{\maximumpicturewidth}{10}
\pgfmathsetmacro{\maximumpictureheight}{5}

\newcommand{\getxyscale}
{   \path (current bounding box.south west);
  \pgfgetlastxy{\xsw}{\ysw}
  \path (current bounding box.north east);
  \pgfgetlastxy{\xne}{\yne}
  \pgfmathsetlengthmacro{\picwidth}{\xne-\xsw}
  \pgfmathsetlengthmacro{\picheight}{\yne-\ysw}
  \pgfmathsetlengthmacro{\maxpicwidth}{\maximumpicturewidth*28.453}
  \pgfmathsetlengthmacro{\maxpicheight}{\maximumpictureheight*28.453}
  \pgfmathsetmacro{\xscale}{\maxpicwidth/\picwidth}
  \pgfmathsetmacro{\yscale}{\maxpicheight/\picheight}
  \xdef\xscalefactor{\xscale}
  \xdef\yscalefactor{\yscale}
}

\NewEnviron{autoscaletikz}%
{\noindent\hphantom{\vphantom{\begin{tikzpicture}\BODY\getxyscale\end{tikzpicture}}}%
\noindent\begin{tikzpicture}[xscale=\xscalefactor,yscale=\yscalefactor]\BODY\end{tikzpicture}}

\title{Continuous Histogram Loss: beyond neural similarity}

\author{
  Artem Zholus \\
  ITMO University \\
  Saint-Petersburg, Russia \\
  \texttt{artem.zholus@gmail.com}
  \And
  Evgeny Putin \\
  ITMO University \\
  Saint-Petersburg, Russia \\
  \texttt{putin.evgeny@gmail.com} \\
}

\begin{document}
\maketitle

\begin{abstract}
Similarity learning has gained a lot of attention from researches in recent years and tons of successful approaches have been recently proposed. However, the majority of the state-of-the-art similarity learning methods consider only a binary similarity. In this paper we introduce a new loss function called Continuous Histogram Loss (CHL) which generalizes recently proposed Histogram loss to multiple-valued similarities, i.e. allowing the acceptable values of similarity to be continuously distributed within some range. The novel loss function is computed by aggregating pairwise distances and similarities into 2D histograms in a differentiable manner and then computing the probability of condition that pairwise distances will not decrease as the similarities increase. The novel loss is capable of solving a wider range of tasks including similarity learning, representation learning and data visualization.
\end{abstract}

\section{Introduction}

Learning deep features plays a crucial role in a wide range of tasks. 
These include: visual search \cite{vis_search}, classification \cite{alex_net}, visualization \cite{tsne2, tsne1, umap}, biometric identification \cite{reidsurvey, triplet2, face_rec}.
Under this approach, the neural network is trained with some loss function, usually classification or some specific one \cite{triplet1} to build a structured space of features. 
The resulting network is then used to calculate embeddings of objects and use them for a particular task. Usually, the algorithm used was explicitly designed for this task.

Despite the diversity of existing approaches, there is still lack of universal approach which could possess an ability to be used in tasks of different kind. 
Our research seeks to create more general algorithm which could be used in wider than it was originally designed for. 

The focus of this work is to build embedding of data with meaningful structure. For this, we generalize existing work on similarity learning \cite{HistLoss}. 
Particularly, we present novel loss function which can accept arbitrary similarities of sample pairs (not just binary) and use it substantially. 
This opens new opportunities to us by allowing to use this loss function in a wider range of tasks.
These tasks include: similarity learning for which our loss can serve as a drop-in replacement for the approach, presented in \cite{HistLoss}, representation learning as the loss can utilize any kind of structured label information and visualization as the loss can bring the structure into low-dimensional feature space.

The loss is based on the approximation of joint density of pairwise distances and associated similarities using differentiable histograms. Once calculated, we estimate the probability that there exist two pairs which violate the natural condition of distances being proportional to similarities. In other words, the probability of a random pair in which objects are more similar to each other than in another one while the former would have a larger distance between objects than the latter.

\section{Related Work}

Methods for learning object representations have been extensively developed in recent years. 
These include similarity learning, representation learning, visualization algorithms. Below we review some of the algorithms and loss functions which are used in these tasks. 

Siamese networks \cite{contr_loss} recently became an ultimate tool for similarity learning. 
For this, the network is trained alongside its shared copy to learn the embedding for objects through the adjustment of relative distances. Modern advancements in this area include triplet based losses \cite{triplet2, triplet0, triplet1}. 
For triplet loss, the set of triplets is constructed from the data where each triplet consists of \textit{anchor} object and its \textit{positive} and \textit{negative} counterparts (according to some similarity). 
Later approaches have used quadruplets \cite{quad_loss}, angles within triplets \cite{ang_loss} or histogram approximations for distance distribution \cite{HistLoss}.
The latter serves as the main approach for us to derive from. 
Tasks, usually solved by similarity learning techniques include person re-identification \cite{reidsurvey}, face recognition \cite{face_rec, triplet2} and visual product search \cite{vis_search}.
Our approach also capable of solving these tasks, however, we do not restrict it to the similarity learning.

One of the most widely used tools for learning data representation are autoencoders. 
The encoder projects each data point to a low dimensional latent space from where the decoder reconstructs the input. 
This simple principle is very powerful as it requires only datapoints for learning meaningful representations. 
Modern variants of AE are able to enforce some distribution in latent space either using the variational inference approach \cite{VAE, VAE2} or using the auxiliary discriminator network \cite{AAE}. 
The work \cite{sae} is most close to ours in the sight of autoencoder based representation learning. 
The authors used an approach which is capable of building informative representations and visualization of the data.

Various representation learning approaches are often employed for visualization purposes. The t-SNE algorithm which is based on Stochastic Neighbor Embedding \cite{tsne1, tsne2} and Uniform Manifold Approximation and Projection (UMAP) \cite{umap} are the examples. The Multidimensional scaling (MDS) \cite{MDS} algorithm is similar to our work since it solves almost the same task of putting the data points into some space according to the pre-defined object similarity. The main difference is that this algorithm learns an embedding in a nonparametric way while we provide a way to amortize this process with some parametric model (e.g. neural network) using the proposed loss function.

\section{Method}
In this section, we describe the proposed loss function which is based on Histogram loss \cite{HistLoss}. Our modification allows it to accept not precisely "positive" and "negative" pairs of objects, but also pairs with arbitrary value of \textit{similarity} between objects.
\subsection{Notation}
Recall batch of training examples $\X = \{x_1, \dots, x_N\}$ and let $S : \X \times \X \rightarrow \{0, 1\}$ be pairwise similarity of object pair from $\X$. Name a pair of objects $(x, y) \in \X\times \X$ to be \textit{positive pair} if $S(x, y) = 1$ and \textit{negative pair} if  $S(x, y) = 0$. Denote deep feed-forward neural network $g$ parameterized by weight vector $\theta$. 
Further, assume $f$ to be some metric over outputs of network $g$. Without loss of generality, we can bound metric $f$ to be in range $[0, 1]$ (e.g. by $f\mapsto \frac{f}{1 + f}$ transformation). For the sake of simplicity, assume that we have some enumeration of pairs of objects $\{x^1_i, x^2_i\}_{i=1}^M, \: x^k_i \in \X$. Recall $i$-th distance as $d_i = f(g(x^1_i), g(x^2_i))$ and similarity $s_i = S(x^1_i, x^2_i)\in\{0,1\}$. Also, we will call $(d, s)$-pair a pair of objects with distance $d$ and similarity $s$.

\subsection{Histogram loss}
In \cite{HistLoss} authors proposed the loss function based on the estimation of the probability distribution of distances between the outputs of neural network. To avoid confusion, we emphasize that original paper have slightly different notation. They use word "similarity" instead of "distance" as their parametrization assumes range $[-1, 1]$ (while we use $[0, 1]$) and outputs $1$, instead of $0$, for pair of same objects. So we change their formulas equivalently according to our parametrization. 

The loss uses histogram approximation using triangular kernel density estimation. To define it, recall $t_i = \frac{i}{R - 1}$ to be nodes for histogram bins and step size $\Delta = \frac{2}{R - 1}$. And 
\begin{gather*}
  \delta_{i,r} =
  \left\{\begin{matrix}
        (d_i - t_{r - 1})/\Delta, && \text{if\:} d_i\in [t_{r-1}, t_r] \\
        (t_{r + 1} - d_i)/\Delta, && \text{if\:} d_i\in [t_{r}, t_{r + 1}] \\
        0, && \text{otherwise}
  \end{matrix}\right.
\end{gather*}

To define the histogram loss, we need to estimate two probability distributions: $p^{+}_{\theta}(x)$ and $p^{-}_{\theta}(x)$ which are the distributions of distances between positive and negative pairs respectively. We then define their histogram approximations as:

\begin{gather*}
    h^+_{r} = \frac{1}{|\mathcal{S}^+|}\sum_{i:s_i = 1} \delta_{i, r}; \:\:\:
    h^-_{r} = \frac{1}{|\mathcal{S}^-|}\sum_{i:s_i = 0} \delta_{i, r}
\end{gather*}

Where $\bs^+ = \{d_i \vert s_i = 1\}$, $\bs^- = \{d_i \vert s_i = 0\}$ are the sets of distances between positive and negative pairs. Histogram loss is equal to the probability of \textit{reverse} meaning that it is defined as the probability of the distance in a random negative pair to be less than the distance in a random positive pair:
\begin{gather}
    p_{\text{reverse}} = \int_0^1 p^-_{\theta}(x)\left(\int_x^1 p^+_{\theta}(y)dy\right)dx = \int_0^1 p^-_{\theta}(x)(1 - \Phi^+_{\theta}(x))dx = \mathbb{E}_{x \sim p^-_{\theta}} [1 - \Phi^+_{\theta}(x)]
\end{gather}
Where $\Phi^+(x)$ is the CDF of $p^+(x)$. Since this expectation is intractable, the authors used its discrete approximation which is calculated as batch-wise histograms.
\begin{gather}
    \HL(\theta) = \sum_{r = 1}^{n}\left(h_r^-\sum_{q = r}^{n}h_q^+\right) =
    \sum_{r = 1}^nh^-_r\phi^+_r
\end{gather}

Importantly, the histogram loss is differentiable w.r.t. the pairwise similarities and hence the model parameters $\theta$:
\begin{gather*}
    \frac{\partial \HL}{\partial h^-_{r}} = \sum_{q = r}^n h^+_q, \:
    \frac{\partial \HL}{\partial h^+_{r}} = \sum_{q = 1}^r h^-_q, \:\:
    \frac{\partial h^+_r}{\partial d_i} = \left\{\begin{matrix}
      \frac{1}{|\mathcal{S}^+|}, && \text{if\:} d_i\in [t_{r-1}, t_r] \\
      -\frac{1}{|\mathcal{S}^+|}, && \text{if\:} d_i\in [t_{r}, t_{r + 1}] \\
      0, && \text{otherwise}
\end{matrix}\right.
\end{gather*}
\subsection{Continuous histogram loss}
In this section we generalize the loss presented in \cite{HistLoss}. 
Our modification allows the network to consider not just strictly positive and negative pair but pairs with a continuously distributed value of $S(x, y)$. So we redefine $S$ as $S: \X\times \X \rightarrow [0, 1]$. Since the similarity became not just abstract property of pair we can reconsider $S$ as the random variable in domain $[0, 1]$. As in the previous section, we are interested in the estimation of the probability of reverse when for a random pair with distance $x$ and similarity $s$ and random $(x', s')$-pair we will have $x > x'$ and $s > s'$ (i.e. in $(x, s)$-pair, objects are more similar but the distance between them is higher than in $(x', s')$-pair). The probability of random $(X, S)$-pair to have distance and similarity greater than particular $x$ and $s$ values is equal to:
\begin{gather*}
    P(X > x, S > s) = \int_{x}^{1}\int_{s}^{1}p(y, t)\text{d}y\text{d}t
\end{gather*}
Having defined that, we can define the probability of reverse as:
\begin{gather}\label{CHL}
    p_{\text{reverse}} = \Exp_{(x, s)\sim p(x, s)} P(X > x, S > s)
\end{gather}

With this formula, we have underestimated the probability of reverse, since there can be another reversion case with the opposite signs inside the expectation of $P(X > x, S > s)$. For the case with opposite signs, the probability $P(X < x, S < s)$ is equal to one of random $(X, S)$-pair with distance less than $x$ and similarity less than $s$. To cover all cases, we should place $P(X > x, S > s) + P(X < x, S < s)$ under the expectation. However, using $P(X > x, S > s)$ is enough since all of the cases have equal gradients, up to multiplicative constant (see Appendix \ref{other_chl} for details).

As before, we approximate the probability of reverse using histograms. We use $n$-dimensional histogram for distance and $m$-dimensional histogram for similarity. Then, each value for each bin in the histogram can be estimated as: 
\begin{gather*}
    h_{r, z} = \frac{1}{M}\sum_{i:\: |s_{i}/\Delta- z| < \frac{1}{2}}\delta_{i,r}
\end{gather*}

Where $r$ and $z$ are distance bins and similarity bins indexers respectively. In the above sum, only those indices $i$ are considered for which $|s_{i}/\Delta- z| < \frac{1}{2}$ meaning that $s_i$ falls into a bin with center in $t_z$. If $s_i$ lies in the boundary of some bin, we assign it to the left bin, we leave it as an implementation detail for simplicity. 
Now, $h_{r, z}$ serves as histogram approximation for $p(x, s)$. Denote approximation for $P(X > x, S > s)$ as $\phi_{rz}$. Since we are using discrete approximation, the corners cases will matter. For later derivations, it will be more convenient to use non-strict inequality for distances (which will not affect true probability). 
Also, for later convenience, we recall $\psi_{rz}$ (not distance-strict) approximation for the probability $P(X < x, S < s)$.
So we set:

\begin{gather*}
    \phi_{rz} = \sum_{r' = r}^{n}\sum_{z' = z + 1}^{m}h_{r',z'}; \:\:\:\:
    \psi_{rz} = \sum_{r' = 1}^{r}\sum_{z' = 1}^{z-1}h_{r',z'}
\end{gather*}

Finally, the expectation (\ref{CHL}) can be computed as:

\begin{gather*}
    L(\theta) = \sum_{r, z}h_{r, z} \phi_{r,z}
\end{gather*}

Where $L$ is the proposed \textit{Continuous Histogram Loss}. 
This loss is differentiable w.r.t. the distances $d_i$. See Appendix \ref{grad} for the explicit form of the gradient. 

\subsection{Relationship to Histogram Loss}

Our loss naturally generalizes \cite{HistLoss} which means that if similarity values are binary, then it leads to the same estimator, up to multiplicative constant. 
To see this, first rewrite positive and negative distributions as the corresponding likelihoods: $p^+(x) = p(x \vert S = 0)$ and $p^-(x) = p(x \vert S = 1)$. Then the probability of reverse is:
\begin{gather*}
    \Exp_{(x, s)\sim p(x, s)} P(X > x, S > s) = \\
    p(S = 0)\Exp_{x\sim p(x\vert S = 0)} \underbrace{P(X > x, S > 0)}_{P(X \geqslant x, S = 1)} + p(S = 1)\Exp_{x\sim p(x\vert S = 1)} \underbrace{P(X > x, S > 1)}_{=0} = \\
    p(S = 0)p(S = 1)\Exp_{x\sim p^-}P(X \geqslant x \vert S = 1) = p(S = 0)p(S = 1)\Exp_{x\sim p^-}[1 - \Phi^+(x))]
\end{gather*}
So, the right part of this equation is the objective for Histogram loss multiplied by $p(S = 0)p(S = 1)$ each of which is just prior probabilities of similarity labels of the dataset thus constants.

\subsection{Conditions for minimum of CHL}\label{min}
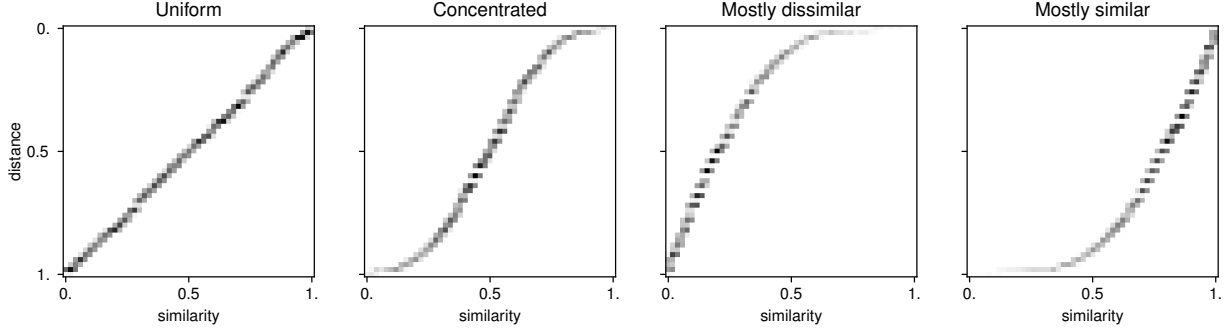
\begin{figure}
    \scalebox{0.6}{\input{figures/toy/fig.pgf}}
    \caption{
        The results of the optimization of distances for different similarity distributions. 
        Each heatmap depicts the joint distribution of distances and similarities of $(d_i, s_i)$-pairs.
        Zero distance corresponds to the closest pair. The similarity which equals to one corresponds to the most similar pair.
    }
    \label{opt_result}
\end{figure}
One of the advantages of the proposed loss function is that it is able to automatically deal with the uncertainty. 
When the number of pairs with different distances but same similarity level exists, there is an ambiguity about where to displace points and how to change distances between them since pairs on other similarity levels haven't displaced yet.

As shown in Appendix \ref{grad}, the gradient of the loss w.r.t. particular distance $d_i$ has the same sign with  $\sum_{z= 0}^{z_i - 1}h_{r_i + 1, z} - \sum_{z= z_i + 1}^{m}h_{r_iz}$, in other words, if bins of the same distance and higher similarity outweigh bins of the next distance level and lower similarity, then the gradient will be positive. 
If they are equal or there are no other pairs within that region of distance and similarity, then the gradient will be equal to zero. 
In the latter case, this is because of the ambiguity --- we are uncertain about whether to decrease the distance or not since we have nothing for that pair to compare to.

This leads to a natural question about the existence of the local minimum of CHL and its form. 
As we showed above, for the case when $\frac{\partial L}{\partial d_i} = 0$, all pairs within one distance level should have the same similarity (i.e. lie in the same similarity bin). 
Moreover, if $f(z) = \arg\max_{r}h_{rz}$ and $f$ is strictly monotonically decreasing function, then $L = 0$. This can be formalized in the following statement.
\begin{prop}
$h$ is local minimum for $L$, yielding $L=0$, if and only if $h$ is such that for each $r$, the unnormalized discrete similarity distribution $h_{r}$ should have zero variance and $f(z) = \arg\max_{r}h_{rz}$ is strictly monotonically decreasing function.
\end{prop}
\begin{proof}
To see this, note that if $f(z_1) \leqslant f(z_2)$ for $z_1 < z_2$, then $h_{f(z_2), z_2}$ will contribute to $\phi_{f(z_1), z_1}$, so $L$ will be above $0$. 
Conversely, if $L = 0$, then for all $r, z$ either $\phi_{rz}$ or $h_{rz}$ equal to $0$.
Take any pair of indices $r, z$. 
For them either $h_{rz} = 0$ and $\phi_{rz} \neq 0$ or $h_{rz} \neq 0$ and $\phi_{rz} = 0$. 
The case $h_{rz} \neq 0$ also entails $\psi_{rz} = 0$ which forces that $h_{rz}$ will be the only non-zero element in $h^T_{z}$. 
Next, note that if $f$ non-descrease its value for increasing $z$, this will violate the assumption and we will get $h_{rz}\phi_{rz} \neq 0$.
\end{proof}


\section{Experiments}
\begin{figure}
\begin{subfigure}{.5\textwidth}
  \centering
  \includegraphics[width=.7\linewidth]{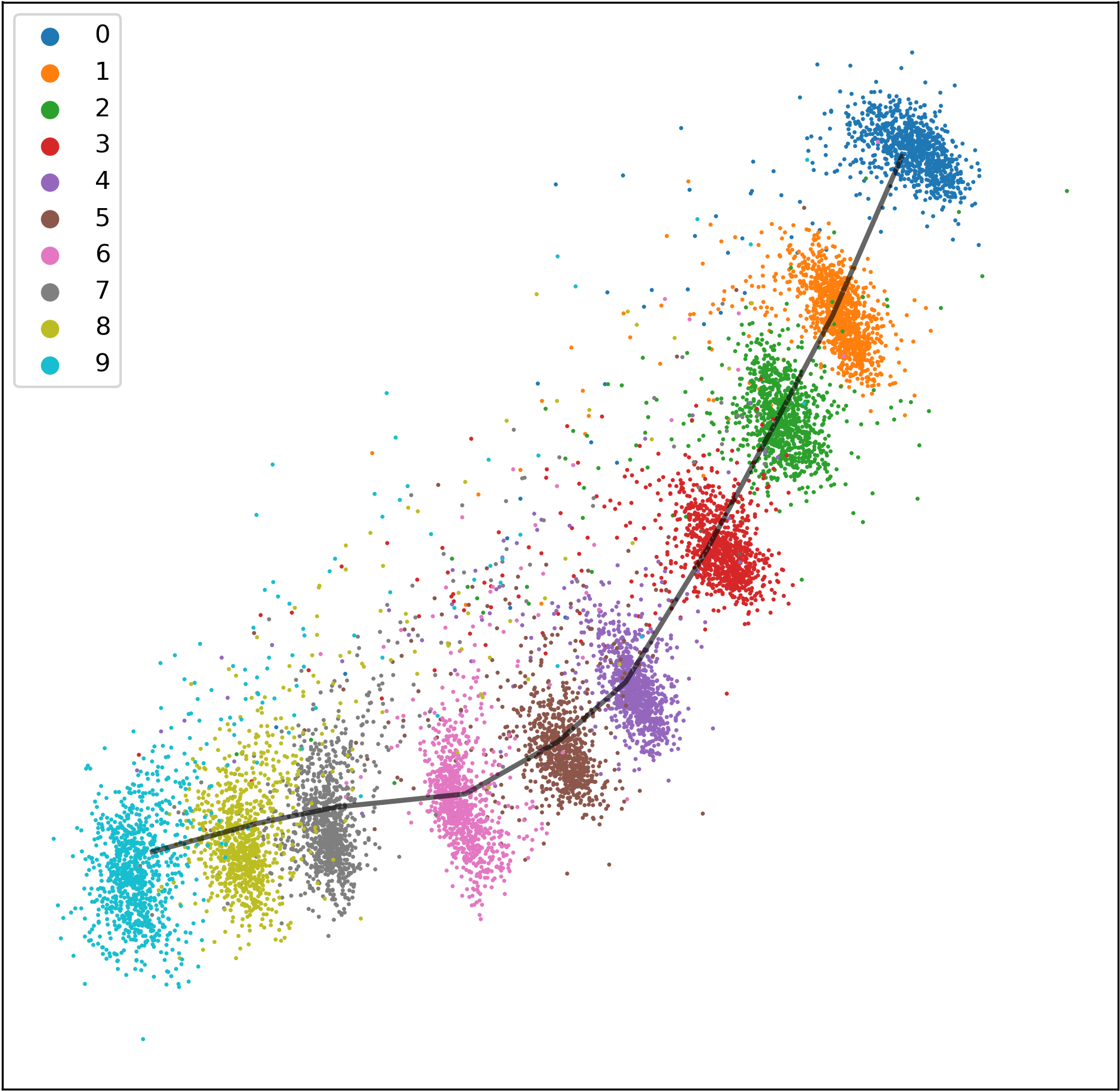}  
  \caption{2D embedding built with CHL}
  \label{fig:chl}
\end{subfigure}
\begin{subfigure}{.5\textwidth}
  \centering
  \includegraphics[width=.7\linewidth]{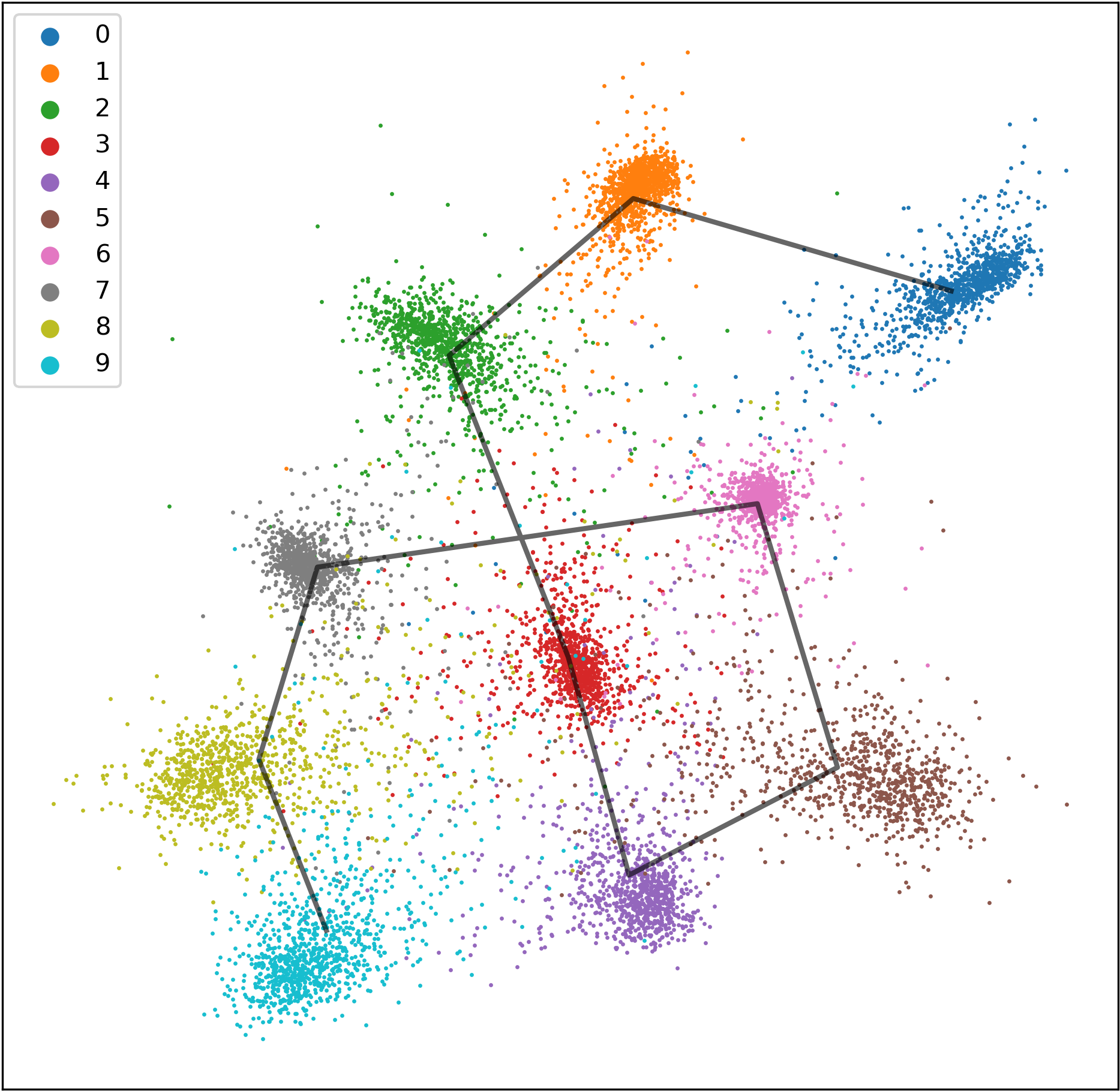}  
  \caption{2D embedding built with HL}
  \label{fig:sub-second}
\end{subfigure}
\caption{Resulting embeddings of test part of MNIST dataset. 
Each image is fed to neural network with 2D output and put on canvas. 
Dots are coloured according to their classes. Centroids of classes of consecutive digits are connected with black line.}
\label{fig:mnist_exp}
\end{figure}

\subsection{Synthetic dataset}\label{synth}
In this experiment, we analyzed the behaviour of the loss function in the simplest form. 
We built dataset as set of pairs of real numbers $\{(d_i, s_i)\}_{i=1}^M$. 
We demonstrate how distance distribution can be learned for different types of prior similarity distributions.  
The 1D distribution is used since we are interested in the quantitative comparison of the learned distance distributions. 
For this we used four different marginal similarity distributions as follows, each simulating a different kind of similarity.
\begin{enumerate}
    \item \textbf{Uniform similarity:} the similarity distribution is uniform with min $0$ and max $1$. 
    \item \textbf{Concentrated similarity:} truncated normal distribution with mean $0.5$ and std $0.3$. 
    Covers the case when similarity values are concentrated around some value and the highest and the lowest similarity pairs are negligible.
    \item \textbf{Mostly dissimilar:} truncated normal distribution with mean $0$ and std $0.3$. 
    Covers the case when a dominant number of pairs have low similarity values, with the number of similar pairs being very low.
    \item \textbf{Mostly similar:} truncated normal distribution with mean $1$ and std $0.3$. 
    Covers the opposite to the previous case when pairs have high similarity.
\end{enumerate}
For each dataset $d_i$ were i.i.d. random variables and $d_i \sim \mathcal{U}(0, 1)$ and $s_i$ were also i.i.d. and sampled from the corresponding distributions. 
Once sampled, the $(d_i, s_i)$-pairs were used to build histograms of size $n= m = 51$. 

For each dataset, the optimization over distances was performed. 
We calculated $\frac{\partial L}{\partial d_i}$ and used it to update $d_i$ values using the vanilla gradient descent updates. 
Learning rate was set to $0.1$ and the optimization was performed for $3000$ iterations. 

The results are presented in Figure \ref{opt_result}. 
Our results show that the resulting distance distribution is highly dependent on the similarity distribution. 
If some region lacks information about similarity values in it, then all pairs within that region will be equidistant. 
Conversely, the more information about similarities between pairs we have, the more confident we are about how to align distances.

More careful consideration of the intermediate optimization steps has shown (see Appendix \ref{anal}) how the loss can deal with uncertainty. According to section \ref{min}, the maximal or minimal distance distribution regions will have a gradient of the largest magnitude. Due to this, similar distances will rapidly condense, forming modes in distribution of distances. Then, points of the opposite sign in gradient will quickly converge to each other and form the curved joint distribution of low entropy. The shape of the curve is dependent on prior similarity distribution.
 
\subsection{MNIST}

In this experiment, we tested the ability of CHL to build meaningful visualizations of the data. 
We used a simple fully connected neural network with 2D output. 
The network was learned with Continuous histogram loss on MNIST \cite{mnist} dataset, with similarity labels equal to $1 - \frac{|i - j|}{10}$ where $i$ and $j$ are class numbers of the corresponding objects in pair. 
The intuition of this similarity is that we want to order encoded objects naturally, according to digit they depict. 
Particularly, we used a fully connected neural network with $2$ hidden layers. 
Layer sizes were $784$-$256$-$128$-$2$ and all activations except last were ELU \cite{ELU}. 
The size of the histogram was set to $100$ for both distance and similarity.
The network was trained using Adam optimizer \cite{Adam} with learning rate $0.002$ for $10$ epochs. 
For comparisons, we also used Histogram loss with binary similarity labels which were set to $1$ in pairs of the same class, otherwise to $0$. 

The results are presented in Figure \ref{fig:mnist_exp}. 
Embeddings built with CHL have captured the information about digit classes and sorted them into meaningful order since the similarity supposed to do. 
But the HL is only able to cluster images according to classes since the similarity is binary and the loss is restricted on the conceptual level.

\section{Conclusion}

In this work, we have proposed a new loss function for learning deep representations and data visualization, called the Continuous histogram loss. 
This loss is inspired by recently proposed Histogram loss and generalizes it for the case of arbitrarily valued similarity.
Unlike the previous work, this loss is able to be used for a larger variety of problems including data visualization, similarity learning and feature learning.

We have demonstrated the quantitative behaviour of the proposed loss function on simple data as well as its visualization capabilities. Conducted experiments on diverse set of tasks have shown the potential of the proposed loss function for further utilization which is subject to the future work on this preprint.

\bibliographystyle{plain}  
\bibliography{references}  
\newpage
\appendix

\section{Derivation of the gradient of Continuous Histogram Loss}\label{grad}
The gradient w.r.t. includes the calculation of proxy derivatives w.r.t. each bin:
\begin{gather}\label{main_grad}
    \frac{\partial L}{\partial d_i} = \sum_{r,z}\frac{\partial L}{\partial h_{r z}}\frac{\partial h_{r z}}{\partial d_{i}}
\end{gather}
To calculate $\frac{\partial h_{rz}}{\partial d_i}$, note that only one index $z_i$ will contribute positively to the derivative w.r.t. $d_i$ ($|s_{i}/\Delta- z_i| < \frac{1}{2}$ for this index) and let $r_i$ be index of current bin $d_i \in [t_{r_i}, t_{r_{i} + 1}]$. 
Next, by linearity of $\delta$, we have:
\begin{gather}\label{h_grad}
    \frac{\partial \delta_{ir}}{\partial d_i} =
  \left\{\begin{matrix}
        \Delta^{-1}, & \text{if\:} d_i\in [t_{r-1}, t_r] \\
        -\Delta^{-1}, & \text{if\:} d_i\in [t_{r}, t_{r + 1}] \\
        0, & \text{otherwise}
  \end{matrix}\right., \:\:\:\:
  \frac{\partial h_{rz}}{\partial d_i} = \left\{\begin{matrix}
  \frac{1}{\Delta N}, & \text{if\:} r = r_i + 1, z = z_i\\
  \frac{-1}{\Delta N}, & \text{if\:} r = r_i, z = z_i\\
  0, & \text{otherwise}
  \end{matrix}\right.
\end{gather}
Second part is $\frac{\partial L}{\partial h_{r, z}}$. Take particular indices $r_0, z_0$. 
The gradient then is:
\begin{gather}\label{lgrad}
    \frac{\partial L}{\partial h_{r_0z_0}} = \frac{\partial}{\partial h_{r_0z_0}}
  \sum_{r, z}h_{r z} \phi_{rz} =
  \sum_{r, z} \frac{\partial}{\partial h_{r_0z_0}}\left(h_{r z} \phi_{rz}\right) =
  \sum_{r, z} \frac{\partial h_{rz}}{\partial h_{r_0z_0}}\phi_{rz} +
  \frac{\partial \phi_{rz}}{\partial h_{r_0z_0}}h_{rz}
\end{gather}
Note that $\frac{\partial\phi_{rz}}{\partial h_{r_0z_0}}$ is either $1$ or $0$ depending on that $r < r_0, z < z_0$ or not. 
Therefore, we have:
\begin{gather}\label{psi}
    \sum_{r, z} \frac{\partial \phi_{rz}}{\partial h_{r_0z_0}}h_{rz} = \sum_{r = 1}^{r_0 - 1}\sum_{z = 1}^{z_0 - 1}h_{rz} = \psi_{r_0 z_0}
\end{gather}
Using equations (\ref{lgrad}) and (\ref{psi}), the gradient is:
\begin{gather}
    \frac{\partial L}{\partial h_{rz}} = \phi_{rz} + \psi_{rz}
\end{gather}
Finally, summing all up gives us:
\begin{gather}
    \frac{\partial L}{\partial d_i} = 
    \sum_{r,z}\frac{\partial L}{\partial h_{r, z}}
    \frac{\partial h_{r, z}}{\partial d_{i}}
    = \sum_{r,z}\phi_{rz}\frac{\partial h_{r, z}}{\partial d_{i}} +
               \psi_{rz}\frac{\partial h_{r, z}}{\partial d_{i}} = \\
  \frac{1}{\Delta M}\left(\phi_{r_i + 1, z_i} - \phi_{r_i, z_i}
                                   + \psi_{r_i + 1, z_i} - \psi_{r_i, z_i}
  \right) = \\
  \frac{1}{\Delta M}\left(\sum_{z = 0}^{z_i - 1}h_{r_i+1, z}- \sum_{z= z_0 + 1}^{m}h_{r_iz}\right)
\end{gather}

\subsection{Other definitions of CHL and their properties}\label{other_chl}
Despite our first definition gives a valid probability of reverse, it still lacks some particular cases. 
For example, we can consider a random pair with distance and similarity \textit{less} than some particular values $x$ and $s$. 
As such, the corresponding probability of reverse will be $\Exp_{(x, s)\sim p(x, s)} P(X < x, S < s)$. 
Or we may cover all cases of inversions by estimating $\Exp_{(x, s)\sim p(x, s)} P(X > x, S > s) + P(X < x, S < s)$. 
However, these definitions have the same gradients (up to multiplicative constant) in their histogram approximations. To show it, let $L_1 = \sum_{r, z}h_{rz} \psi_{rz}$ and $L_2 = L + L_1$.
\begin{gather*}
    \frac{\partial L_1}{\partial h_{r_0z_0}} = 
  \sum_{r, z} \frac{\partial h_{rz}}{\partial h_{r_0z_0}}\psi_{rz} +
  \frac{\partial \psi_{rz}}{\partial h_{r_0z_0}}h_{rz}, \:\:\:\:
  \sum_{r, z} \frac{\partial \psi_{rz}}{\partial h_{r_0z_0}}h_{rz} =
  \phi_{r_0z_0}
\end{gather*}
Which implies that $\frac{\partial L_1}{\partial h_{rz}} = \psi_{rz} + \phi_{rz}$ and $\frac{\partial L_2}{\partial h_{rz}} = 2(\psi_{rz} + \phi_{rz})$.

\section{Visualization of the optimization process}\label{anal}
Here, we provide more detailed visualization of optimization process described in the section \ref{synth}. 
The visualization is presented in Figure \ref{big_pic}.
\begin{figure}
    \centering
    \scalebox{0.8}{\input{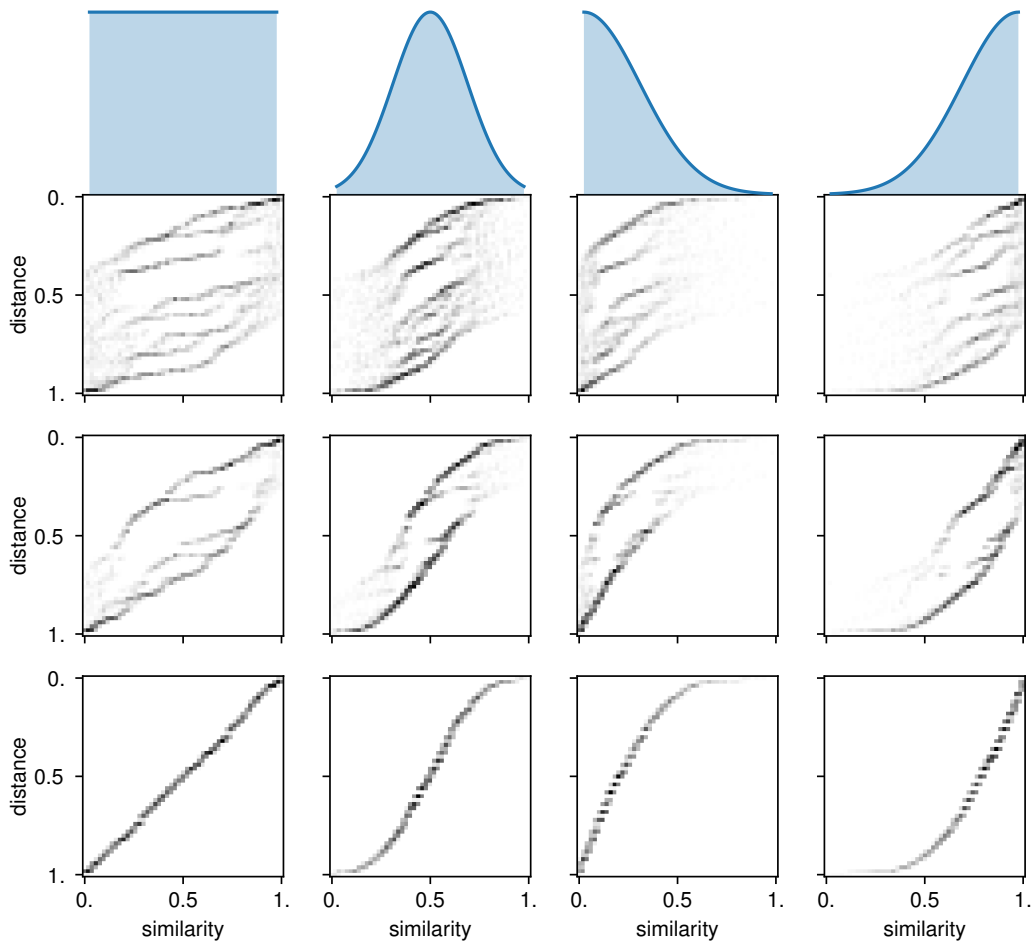}}
    \caption{
        The optimization process of CHL in the space of distances. 
        \textbf{Rirst row:} prior distributions of similarities (from left to right):
        Uniform, truncated normal with mean $0.5$ and std $0.3$, same distribution with mean $0$ (mostly dissimilar objects), same distribution with mean $1$ (mostly similar objects). 
        \textbf{Second, third and bottom row:} every row corresponds to a particular intermediate step of the optimization process. 
        Step numbers are $500$, $1000$, $3000$ respectively. 
        Each heatmap depicts joint distribution of distance and similarity for each step and distribution.
    }
    \label{big_pic}
\end{figure}
\end{document}

%% file: figures/toy/fig.pgf
\begingroup%
\makeatletter%
\begin{pgfpicture}%
\pgfpathrectangle{\pgfpointorigin}{\pgfqpoint{10.627854in}{2.821782in}}%
\pgfusepath{use as bounding box, clip}%
\begin{pgfscope}%
\pgfsetbuttcap%
\pgfsetmiterjoin%
\definecolor{currentfill}{rgb}{1.000000,1.000000,1.000000}%
\pgfsetfillcolor{currentfill}%
\pgfsetlinewidth{0.000000pt}%
\definecolor{currentstroke}{rgb}{1.000000,1.000000,1.000000}%
\pgfsetstrokecolor{currentstroke}%
\pgfsetdash{}{0pt}%
\pgfpathmoveto{\pgfqpoint{0.000000in}{0.000000in}}%
\pgfpathlineto{\pgfqpoint{10.627854in}{0.000000in}}%
\pgfpathlineto{\pgfqpoint{10.627854in}{2.821782in}}%
\pgfpathlineto{\pgfqpoint{0.000000in}{2.821782in}}%
\pgfpathclose%
\pgfusepath{fill}%
\end{pgfscope}%
\begin{pgfscope}%
\pgfsetbuttcap%
\pgfsetmiterjoin%
\definecolor{currentfill}{rgb}{1.000000,1.000000,1.000000}%
\pgfsetfillcolor{currentfill}%
\pgfsetlinewidth{0.000000pt}%
\definecolor{currentstroke}{rgb}{0.000000,0.000000,0.000000}%
\pgfsetstrokecolor{currentstroke}%
\pgfsetstrokeopacity{0.000000}%
\pgfsetdash{}{0pt}%
\pgfpathmoveto{\pgfqpoint{0.508070in}{0.421603in}}%
\pgfpathlineto{\pgfqpoint{2.698287in}{0.421603in}}%
\pgfpathlineto{\pgfqpoint{2.698287in}{2.611821in}}%
\pgfpathlineto{\pgfqpoint{0.508070in}{2.611821in}}%
\pgfpathclose%
\pgfusepath{fill}%
\end{pgfscope}%
\begin{pgfscope}%
\pgfpathrectangle{\pgfqpoint{0.508070in}{0.421603in}}{\pgfqpoint{2.190217in}{2.190217in}}%
\pgfusepath{clip}%
\pgfsys@transformshift{0.508070in}{0.421603in}%
\pgftext[left,bottom]{\pgfimage[interpolate=true,width=2.193333in,height=2.193333in]{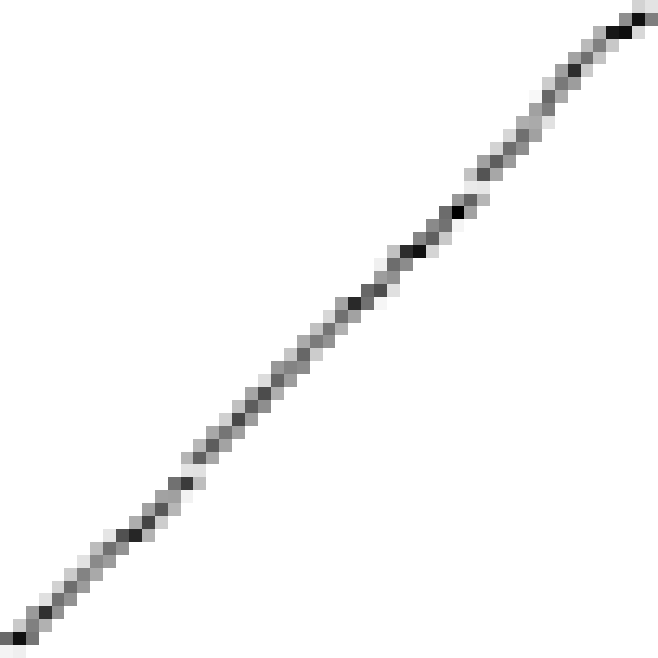}}%
\end{pgfscope}%
\begin{pgfscope}%
\pgfsetbuttcap%
\pgfsetroundjoin%
\definecolor{currentfill}{rgb}{0.000000,0.000000,0.000000}%
\pgfsetfillcolor{currentfill}%
\pgfsetlinewidth{0.803000pt}%
\definecolor{currentstroke}{rgb}{0.000000,0.000000,0.000000}%
\pgfsetstrokecolor{currentstroke}%
\pgfsetdash{}{0pt}%
\pgfsys@defobject{currentmarker}{\pgfqpoint{0.000000in}{-0.048611in}}{\pgfqpoint{0.000000in}{0.000000in}}{%
\pgfpathmoveto{\pgfqpoint{0.000000in}{0.000000in}}%
\pgfpathlineto{\pgfqpoint{0.000000in}{-0.048611in}}%
\pgfusepath{stroke,fill}%
}%
\begin{pgfscope}%
\pgfsys@transformshift{0.529543in}{0.421603in}%
\pgfsys@useobject{currentmarker}{}%
\end{pgfscope}%
\end{pgfscope}%
\begin{pgfscope}%
\definecolor{textcolor}{rgb}{0.000000,0.000000,0.000000}%
\pgfsetstrokecolor{textcolor}%
\pgfsetfillcolor{textcolor}%
\pgftext[x=0.529543in,y=0.324381in,,top]{\color{textcolor}\sffamily\fontsize{10.000000}{12.000000}\selectfont 0.}%
\end{pgfscope}%
\begin{pgfscope}%
\pgfsetbuttcap%
\pgfsetroundjoin%
\definecolor{currentfill}{rgb}{0.000000,0.000000,0.000000}%
\pgfsetfillcolor{currentfill}%
\pgfsetlinewidth{0.803000pt}%
\definecolor{currentstroke}{rgb}{0.000000,0.000000,0.000000}%
\pgfsetstrokecolor{currentstroke}%
\pgfsetdash{}{0pt}%
\pgfsys@defobject{currentmarker}{\pgfqpoint{0.000000in}{-0.048611in}}{\pgfqpoint{0.000000in}{0.000000in}}{%
\pgfpathmoveto{\pgfqpoint{0.000000in}{0.000000in}}%
\pgfpathlineto{\pgfqpoint{0.000000in}{-0.048611in}}%
\pgfusepath{stroke,fill}%
}%
\begin{pgfscope}%
\pgfsys@transformshift{1.603179in}{0.421603in}%
\pgfsys@useobject{currentmarker}{}%
\end{pgfscope}%
\end{pgfscope}%
\begin{pgfscope}%
\definecolor{textcolor}{rgb}{0.000000,0.000000,0.000000}%
\pgfsetstrokecolor{textcolor}%
\pgfsetfillcolor{textcolor}%
\pgftext[x=1.603179in,y=0.324381in,,top]{\color{textcolor}\sffamily\fontsize{10.000000}{12.000000}\selectfont 0.5}%
\end{pgfscope}%
\begin{pgfscope}%
\pgfsetbuttcap%
\pgfsetroundjoin%
\definecolor{currentfill}{rgb}{0.000000,0.000000,0.000000}%
\pgfsetfillcolor{currentfill}%
\pgfsetlinewidth{0.803000pt}%
\definecolor{currentstroke}{rgb}{0.000000,0.000000,0.000000}%
\pgfsetstrokecolor{currentstroke}%
\pgfsetdash{}{0pt}%
\pgfsys@defobject{currentmarker}{\pgfqpoint{0.000000in}{-0.048611in}}{\pgfqpoint{0.000000in}{0.000000in}}{%
\pgfpathmoveto{\pgfqpoint{0.000000in}{0.000000in}}%
\pgfpathlineto{\pgfqpoint{0.000000in}{-0.048611in}}%
\pgfusepath{stroke,fill}%
}%
\begin{pgfscope}%
\pgfsys@transformshift{2.676815in}{0.421603in}%
\pgfsys@useobject{currentmarker}{}%
\end{pgfscope}%
\end{pgfscope}%
\begin{pgfscope}%
\definecolor{textcolor}{rgb}{0.000000,0.000000,0.000000}%
\pgfsetstrokecolor{textcolor}%
\pgfsetfillcolor{textcolor}%
\pgftext[x=2.676815in,y=0.324381in,,top]{\color{textcolor}\sffamily\fontsize{10.000000}{12.000000}\selectfont 1.}%
\end{pgfscope}%
\begin{pgfscope}%
\definecolor{textcolor}{rgb}{0.000000,0.000000,0.000000}%
\pgfsetstrokecolor{textcolor}%
\pgfsetfillcolor{textcolor}%
\pgftext[x=1.603179in,y=0.134413in,,top]{\color{textcolor}\sffamily\fontsize{10.000000}{12.000000}\selectfont similarity}%
\end{pgfscope}%
\begin{pgfscope}%
\pgfsetbuttcap%
\pgfsetroundjoin%
\definecolor{currentfill}{rgb}{0.000000,0.000000,0.000000}%
\pgfsetfillcolor{currentfill}%
\pgfsetlinewidth{0.803000pt}%
\definecolor{currentstroke}{rgb}{0.000000,0.000000,0.000000}%
\pgfsetstrokecolor{currentstroke}%
\pgfsetdash{}{0pt}%
\pgfsys@defobject{currentmarker}{\pgfqpoint{-0.048611in}{0.000000in}}{\pgfqpoint{0.000000in}{0.000000in}}{%
\pgfpathmoveto{\pgfqpoint{0.000000in}{0.000000in}}%
\pgfpathlineto{\pgfqpoint{-0.048611in}{0.000000in}}%
\pgfusepath{stroke,fill}%
}%
\begin{pgfscope}%
\pgfsys@transformshift{0.508070in}{2.590348in}%
\pgfsys@useobject{currentmarker}{}%
\end{pgfscope}%
\end{pgfscope}%
\begin{pgfscope}%
\definecolor{textcolor}{rgb}{0.000000,0.000000,0.000000}%
\pgfsetstrokecolor{textcolor}%
\pgfsetfillcolor{textcolor}%
\pgftext[x=0.278334in,y=2.537586in,left,base]{\color{textcolor}\sffamily\fontsize{10.000000}{12.000000}\selectfont 0.}%
\end{pgfscope}%
\begin{pgfscope}%
\pgfsetbuttcap%
\pgfsetroundjoin%
\definecolor{currentfill}{rgb}{0.000000,0.000000,0.000000}%
\pgfsetfillcolor{currentfill}%
\pgfsetlinewidth{0.803000pt}%
\definecolor{currentstroke}{rgb}{0.000000,0.000000,0.000000}%
\pgfsetstrokecolor{currentstroke}%
\pgfsetdash{}{0pt}%
\pgfsys@defobject{currentmarker}{\pgfqpoint{-0.048611in}{0.000000in}}{\pgfqpoint{0.000000in}{0.000000in}}{%
\pgfpathmoveto{\pgfqpoint{0.000000in}{0.000000in}}%
\pgfpathlineto{\pgfqpoint{-0.048611in}{0.000000in}}%
\pgfusepath{stroke,fill}%
}%
\begin{pgfscope}%
\pgfsys@transformshift{0.508070in}{1.516712in}%
\pgfsys@useobject{currentmarker}{}%
\end{pgfscope}%
\end{pgfscope}%
\begin{pgfscope}%
\definecolor{textcolor}{rgb}{0.000000,0.000000,0.000000}%
\pgfsetstrokecolor{textcolor}%
\pgfsetfillcolor{textcolor}%
\pgftext[x=0.189968in,y=1.463951in,left,base]{\color{textcolor}\sffamily\fontsize{10.000000}{12.000000}\selectfont 0.5}%
\end{pgfscope}%
\begin{pgfscope}%
\pgfsetbuttcap%
\pgfsetroundjoin%
\definecolor{currentfill}{rgb}{0.000000,0.000000,0.000000}%
\pgfsetfillcolor{currentfill}%
\pgfsetlinewidth{0.803000pt}%
\definecolor{currentstroke}{rgb}{0.000000,0.000000,0.000000}%
\pgfsetstrokecolor{currentstroke}%
\pgfsetdash{}{0pt}%
\pgfsys@defobject{currentmarker}{\pgfqpoint{-0.048611in}{0.000000in}}{\pgfqpoint{0.000000in}{0.000000in}}{%
\pgfpathmoveto{\pgfqpoint{0.000000in}{0.000000in}}%
\pgfpathlineto{\pgfqpoint{-0.048611in}{0.000000in}}%
\pgfusepath{stroke,fill}%
}%
\begin{pgfscope}%
\pgfsys@transformshift{0.508070in}{0.443076in}%
\pgfsys@useobject{currentmarker}{}%
\end{pgfscope}%
\end{pgfscope}%
\begin{pgfscope}%
\definecolor{textcolor}{rgb}{0.000000,0.000000,0.000000}%
\pgfsetstrokecolor{textcolor}%
\pgfsetfillcolor{textcolor}%
\pgftext[x=0.278334in,y=0.390315in,left,base]{\color{textcolor}\sffamily\fontsize{10.000000}{12.000000}\selectfont 1.}%
\end{pgfscope}%
\begin{pgfscope}%
\definecolor{textcolor}{rgb}{0.000000,0.000000,0.000000}%
\pgfsetstrokecolor{textcolor}%
\pgfsetfillcolor{textcolor}%
\pgftext[x=0.134413in,y=1.516712in,,bottom,rotate=90.000000]{\color{textcolor}\sffamily\fontsize{10.000000}{12.000000}\selectfont distance}%
\end{pgfscope}%
\begin{pgfscope}%
\pgfsetrectcap%
\pgfsetmiterjoin%
\pgfsetlinewidth{0.803000pt}%
\definecolor{currentstroke}{rgb}{0.000000,0.000000,0.000000}%
\pgfsetstrokecolor{currentstroke}%
\pgfsetdash{}{0pt}%
\pgfpathmoveto{\pgfqpoint{0.508070in}{0.421603in}}%
\pgfpathlineto{\pgfqpoint{0.508070in}{2.611821in}}%
\pgfusepath{stroke}%
\end{pgfscope}%
\begin{pgfscope}%
\pgfsetrectcap%
\pgfsetmiterjoin%
\pgfsetlinewidth{0.803000pt}%
\definecolor{currentstroke}{rgb}{0.000000,0.000000,0.000000}%
\pgfsetstrokecolor{currentstroke}%
\pgfsetdash{}{0pt}%
\pgfpathmoveto{\pgfqpoint{2.698287in}{0.421603in}}%
\pgfpathlineto{\pgfqpoint{2.698287in}{2.611821in}}%
\pgfusepath{stroke}%
\end{pgfscope}%
\begin{pgfscope}%
\pgfsetrectcap%
\pgfsetmiterjoin%
\pgfsetlinewidth{0.803000pt}%
\definecolor{currentstroke}{rgb}{0.000000,0.000000,0.000000}%
\pgfsetstrokecolor{currentstroke}%
\pgfsetdash{}{0pt}%
\pgfpathmoveto{\pgfqpoint{0.508070in}{0.421603in}}%
\pgfpathlineto{\pgfqpoint{2.698287in}{0.421603in}}%
\pgfusepath{stroke}%
\end{pgfscope}%
\begin{pgfscope}%
\pgfsetrectcap%
\pgfsetmiterjoin%
\pgfsetlinewidth{0.803000pt}%
\definecolor{currentstroke}{rgb}{0.000000,0.000000,0.000000}%
\pgfsetstrokecolor{currentstroke}%
\pgfsetdash{}{0pt}%
\pgfpathmoveto{\pgfqpoint{0.508070in}{2.611821in}}%
\pgfpathlineto{\pgfqpoint{2.698287in}{2.611821in}}%
\pgfusepath{stroke}%
\end{pgfscope}%
\begin{pgfscope}%
\definecolor{textcolor}{rgb}{0.000000,0.000000,0.000000}%
\pgfsetstrokecolor{textcolor}%
\pgfsetfillcolor{textcolor}%
\pgftext[x=1.603179in,y=2.695154in,,base]{\color{textcolor}\sffamily\fontsize{12.000000}{14.400000}\selectfont Uniform}%
\end{pgfscope}%
\begin{pgfscope}%
\pgfsetbuttcap%
\pgfsetmiterjoin%
\definecolor{currentfill}{rgb}{1.000000,1.000000,1.000000}%
\pgfsetfillcolor{currentfill}%
\pgfsetlinewidth{0.000000pt}%
\definecolor{currentstroke}{rgb}{0.000000,0.000000,0.000000}%
\pgfsetstrokecolor{currentstroke}%
\pgfsetstrokeopacity{0.000000}%
\pgfsetdash{}{0pt}%
\pgfpathmoveto{\pgfqpoint{3.136331in}{0.421603in}}%
\pgfpathlineto{\pgfqpoint{5.326548in}{0.421603in}}%
\pgfpathlineto{\pgfqpoint{5.326548in}{2.611821in}}%
\pgfpathlineto{\pgfqpoint{3.136331in}{2.611821in}}%
\pgfpathclose%
\pgfusepath{fill}%
\end{pgfscope}%
\begin{pgfscope}%
\pgfpathrectangle{\pgfqpoint{3.136331in}{0.421603in}}{\pgfqpoint{2.190217in}{2.190217in}}%
\pgfusepath{clip}%
\pgfsys@transformshift{3.136331in}{0.421603in}%
\pgftext[left,bottom]{\pgfimage[interpolate=true,width=2.193333in,height=2.193333in]{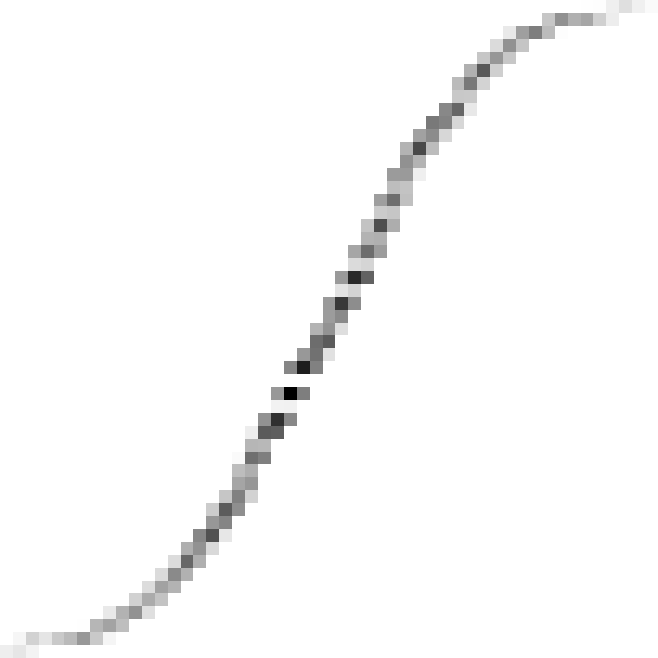}}%
\end{pgfscope}%
\begin{pgfscope}%
\pgfsetbuttcap%
\pgfsetroundjoin%
\definecolor{currentfill}{rgb}{0.000000,0.000000,0.000000}%
\pgfsetfillcolor{currentfill}%
\pgfsetlinewidth{0.803000pt}%
\definecolor{currentstroke}{rgb}{0.000000,0.000000,0.000000}%
\pgfsetstrokecolor{currentstroke}%
\pgfsetdash{}{0pt}%
\pgfsys@defobject{currentmarker}{\pgfqpoint{0.000000in}{-0.048611in}}{\pgfqpoint{0.000000in}{0.000000in}}{%
\pgfpathmoveto{\pgfqpoint{0.000000in}{0.000000in}}%
\pgfpathlineto{\pgfqpoint{0.000000in}{-0.048611in}}%
\pgfusepath{stroke,fill}%
}%
\begin{pgfscope}%
\pgfsys@transformshift{3.157804in}{0.421603in}%
\pgfsys@useobject{currentmarker}{}%
\end{pgfscope}%
\end{pgfscope}%
\begin{pgfscope}%
\definecolor{textcolor}{rgb}{0.000000,0.000000,0.000000}%
\pgfsetstrokecolor{textcolor}%
\pgfsetfillcolor{textcolor}%
\pgftext[x=3.157804in,y=0.324381in,,top]{\color{textcolor}\sffamily\fontsize{10.000000}{12.000000}\selectfont 0.}%
\end{pgfscope}%
\begin{pgfscope}%
\pgfsetbuttcap%
\pgfsetroundjoin%
\definecolor{currentfill}{rgb}{0.000000,0.000000,0.000000}%
\pgfsetfillcolor{currentfill}%
\pgfsetlinewidth{0.803000pt}%
\definecolor{currentstroke}{rgb}{0.000000,0.000000,0.000000}%
\pgfsetstrokecolor{currentstroke}%
\pgfsetdash{}{0pt}%
\pgfsys@defobject{currentmarker}{\pgfqpoint{0.000000in}{-0.048611in}}{\pgfqpoint{0.000000in}{0.000000in}}{%
\pgfpathmoveto{\pgfqpoint{0.000000in}{0.000000in}}%
\pgfpathlineto{\pgfqpoint{0.000000in}{-0.048611in}}%
\pgfusepath{stroke,fill}%
}%
\begin{pgfscope}%
\pgfsys@transformshift{4.231440in}{0.421603in}%
\pgfsys@useobject{currentmarker}{}%
\end{pgfscope}%
\end{pgfscope}%
\begin{pgfscope}%
\definecolor{textcolor}{rgb}{0.000000,0.000000,0.000000}%
\pgfsetstrokecolor{textcolor}%
\pgfsetfillcolor{textcolor}%
\pgftext[x=4.231440in,y=0.324381in,,top]{\color{textcolor}\sffamily\fontsize{10.000000}{12.000000}\selectfont 0.5}%
\end{pgfscope}%
\begin{pgfscope}%
\pgfsetbuttcap%
\pgfsetroundjoin%
\definecolor{currentfill}{rgb}{0.000000,0.000000,0.000000}%
\pgfsetfillcolor{currentfill}%
\pgfsetlinewidth{0.803000pt}%
\definecolor{currentstroke}{rgb}{0.000000,0.000000,0.000000}%
\pgfsetstrokecolor{currentstroke}%
\pgfsetdash{}{0pt}%
\pgfsys@defobject{currentmarker}{\pgfqpoint{0.000000in}{-0.048611in}}{\pgfqpoint{0.000000in}{0.000000in}}{%
\pgfpathmoveto{\pgfqpoint{0.000000in}{0.000000in}}%
\pgfpathlineto{\pgfqpoint{0.000000in}{-0.048611in}}%
\pgfusepath{stroke,fill}%
}%
\begin{pgfscope}%
\pgfsys@transformshift{5.305076in}{0.421603in}%
\pgfsys@useobject{currentmarker}{}%
\end{pgfscope}%
\end{pgfscope}%
\begin{pgfscope}%
\definecolor{textcolor}{rgb}{0.000000,0.000000,0.000000}%
\pgfsetstrokecolor{textcolor}%
\pgfsetfillcolor{textcolor}%
\pgftext[x=5.305076in,y=0.324381in,,top]{\color{textcolor}\sffamily\fontsize{10.000000}{12.000000}\selectfont 1.}%
\end{pgfscope}%
\begin{pgfscope}%
\definecolor{textcolor}{rgb}{0.000000,0.000000,0.000000}%
\pgfsetstrokecolor{textcolor}%
\pgfsetfillcolor{textcolor}%
\pgftext[x=4.231440in,y=0.134413in,,top]{\color{textcolor}\sffamily\fontsize{10.000000}{12.000000}\selectfont similarity}%
\end{pgfscope}%
\begin{pgfscope}%
\pgfsetbuttcap%
\pgfsetroundjoin%
\definecolor{currentfill}{rgb}{0.000000,0.000000,0.000000}%
\pgfsetfillcolor{currentfill}%
\pgfsetlinewidth{0.803000pt}%
\definecolor{currentstroke}{rgb}{0.000000,0.000000,0.000000}%
\pgfsetstrokecolor{currentstroke}%
\pgfsetdash{}{0pt}%
\pgfsys@defobject{currentmarker}{\pgfqpoint{-0.048611in}{0.000000in}}{\pgfqpoint{0.000000in}{0.000000in}}{%
\pgfpathmoveto{\pgfqpoint{0.000000in}{0.000000in}}%
\pgfpathlineto{\pgfqpoint{-0.048611in}{0.000000in}}%
\pgfusepath{stroke,fill}%
}%
\begin{pgfscope}%
\pgfsys@transformshift{3.136331in}{2.590348in}%
\pgfsys@useobject{currentmarker}{}%
\end{pgfscope}%
\end{pgfscope}%
\begin{pgfscope}%
\pgfsetbuttcap%
\pgfsetroundjoin%
\definecolor{currentfill}{rgb}{0.000000,0.000000,0.000000}%
\pgfsetfillcolor{currentfill}%
\pgfsetlinewidth{0.803000pt}%
\definecolor{currentstroke}{rgb}{0.000000,0.000000,0.000000}%
\pgfsetstrokecolor{currentstroke}%
\pgfsetdash{}{0pt}%
\pgfsys@defobject{currentmarker}{\pgfqpoint{-0.048611in}{0.000000in}}{\pgfqpoint{0.000000in}{0.000000in}}{%
\pgfpathmoveto{\pgfqpoint{0.000000in}{0.000000in}}%
\pgfpathlineto{\pgfqpoint{-0.048611in}{0.000000in}}%
\pgfusepath{stroke,fill}%
}%
\begin{pgfscope}%
\pgfsys@transformshift{3.136331in}{1.516712in}%
\pgfsys@useobject{currentmarker}{}%
\end{pgfscope}%
\end{pgfscope}%
\begin{pgfscope}%
\pgfsetbuttcap%
\pgfsetroundjoin%
\definecolor{currentfill}{rgb}{0.000000,0.000000,0.000000}%
\pgfsetfillcolor{currentfill}%
\pgfsetlinewidth{0.803000pt}%
\definecolor{currentstroke}{rgb}{0.000000,0.000000,0.000000}%
\pgfsetstrokecolor{currentstroke}%
\pgfsetdash{}{0pt}%
\pgfsys@defobject{currentmarker}{\pgfqpoint{-0.048611in}{0.000000in}}{\pgfqpoint{0.000000in}{0.000000in}}{%
\pgfpathmoveto{\pgfqpoint{0.000000in}{0.000000in}}%
\pgfpathlineto{\pgfqpoint{-0.048611in}{0.000000in}}%
\pgfusepath{stroke,fill}%
}%
\begin{pgfscope}%
\pgfsys@transformshift{3.136331in}{0.443076in}%
\pgfsys@useobject{currentmarker}{}%
\end{pgfscope}%
\end{pgfscope}%
\begin{pgfscope}%
\pgfsetrectcap%
\pgfsetmiterjoin%
\pgfsetlinewidth{0.803000pt}%
\definecolor{currentstroke}{rgb}{0.000000,0.000000,0.000000}%
\pgfsetstrokecolor{currentstroke}%
\pgfsetdash{}{0pt}%
\pgfpathmoveto{\pgfqpoint{3.136331in}{0.421603in}}%
\pgfpathlineto{\pgfqpoint{3.136331in}{2.611821in}}%
\pgfusepath{stroke}%
\end{pgfscope}%
\begin{pgfscope}%
\pgfsetrectcap%
\pgfsetmiterjoin%
\pgfsetlinewidth{0.803000pt}%
\definecolor{currentstroke}{rgb}{0.000000,0.000000,0.000000}%
\pgfsetstrokecolor{currentstroke}%
\pgfsetdash{}{0pt}%
\pgfpathmoveto{\pgfqpoint{5.326548in}{0.421603in}}%
\pgfpathlineto{\pgfqpoint{5.326548in}{2.611821in}}%
\pgfusepath{stroke}%
\end{pgfscope}%
\begin{pgfscope}%
\pgfsetrectcap%
\pgfsetmiterjoin%
\pgfsetlinewidth{0.803000pt}%
\definecolor{currentstroke}{rgb}{0.000000,0.000000,0.000000}%
\pgfsetstrokecolor{currentstroke}%
\pgfsetdash{}{0pt}%
\pgfpathmoveto{\pgfqpoint{3.136331in}{0.421603in}}%
\pgfpathlineto{\pgfqpoint{5.326548in}{0.421603in}}%
\pgfusepath{stroke}%
\end{pgfscope}%
\begin{pgfscope}%
\pgfsetrectcap%
\pgfsetmiterjoin%
\pgfsetlinewidth{0.803000pt}%
\definecolor{currentstroke}{rgb}{0.000000,0.000000,0.000000}%
\pgfsetstrokecolor{currentstroke}%
\pgfsetdash{}{0pt}%
\pgfpathmoveto{\pgfqpoint{3.136331in}{2.611821in}}%
\pgfpathlineto{\pgfqpoint{5.326548in}{2.611821in}}%
\pgfusepath{stroke}%
\end{pgfscope}%
\begin{pgfscope}%
\definecolor{textcolor}{rgb}{0.000000,0.000000,0.000000}%
\pgfsetstrokecolor{textcolor}%
\pgfsetfillcolor{textcolor}%
\pgftext[x=4.231440in,y=2.695154in,,base]{\color{textcolor}\sffamily\fontsize{12.000000}{14.400000}\selectfont Concentrated}%
\end{pgfscope}%
\begin{pgfscope}%
\pgfsetbuttcap%
\pgfsetmiterjoin%
\definecolor{currentfill}{rgb}{1.000000,1.000000,1.000000}%
\pgfsetfillcolor{currentfill}%
\pgfsetlinewidth{0.000000pt}%
\definecolor{currentstroke}{rgb}{0.000000,0.000000,0.000000}%
\pgfsetstrokecolor{currentstroke}%
\pgfsetstrokeopacity{0.000000}%
\pgfsetdash{}{0pt}%
\pgfpathmoveto{\pgfqpoint{5.764592in}{0.421603in}}%
\pgfpathlineto{\pgfqpoint{7.954809in}{0.421603in}}%
\pgfpathlineto{\pgfqpoint{7.954809in}{2.611821in}}%
\pgfpathlineto{\pgfqpoint{5.764592in}{2.611821in}}%
\pgfpathclose%
\pgfusepath{fill}%
\end{pgfscope}%
\begin{pgfscope}%
\pgfpathrectangle{\pgfqpoint{5.764592in}{0.421603in}}{\pgfqpoint{2.190217in}{2.190217in}}%
\pgfusepath{clip}%
\pgfsys@transformshift{5.764592in}{0.421603in}%
\pgftext[left,bottom]{\pgfimage[interpolate=true,width=2.193333in,height=2.193333in]{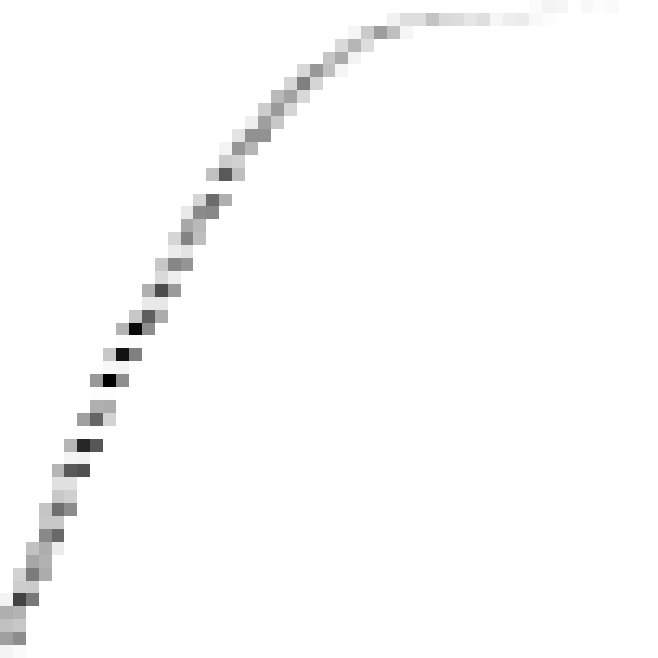}}%
\end{pgfscope}%
\begin{pgfscope}%
\pgfsetbuttcap%
\pgfsetroundjoin%
\definecolor{currentfill}{rgb}{0.000000,0.000000,0.000000}%
\pgfsetfillcolor{currentfill}%
\pgfsetlinewidth{0.803000pt}%
\definecolor{currentstroke}{rgb}{0.000000,0.000000,0.000000}%
\pgfsetstrokecolor{currentstroke}%
\pgfsetdash{}{0pt}%
\pgfsys@defobject{currentmarker}{\pgfqpoint{0.000000in}{-0.048611in}}{\pgfqpoint{0.000000in}{0.000000in}}{%
\pgfpathmoveto{\pgfqpoint{0.000000in}{0.000000in}}%
\pgfpathlineto{\pgfqpoint{0.000000in}{-0.048611in}}%
\pgfusepath{stroke,fill}%
}%
\begin{pgfscope}%
\pgfsys@transformshift{5.786064in}{0.421603in}%
\pgfsys@useobject{currentmarker}{}%
\end{pgfscope}%
\end{pgfscope}%
\begin{pgfscope}%
\definecolor{textcolor}{rgb}{0.000000,0.000000,0.000000}%
\pgfsetstrokecolor{textcolor}%
\pgfsetfillcolor{textcolor}%
\pgftext[x=5.786064in,y=0.324381in,,top]{\color{textcolor}\sffamily\fontsize{10.000000}{12.000000}\selectfont 0.}%
\end{pgfscope}%
\begin{pgfscope}%
\pgfsetbuttcap%
\pgfsetroundjoin%
\definecolor{currentfill}{rgb}{0.000000,0.000000,0.000000}%
\pgfsetfillcolor{currentfill}%
\pgfsetlinewidth{0.803000pt}%
\definecolor{currentstroke}{rgb}{0.000000,0.000000,0.000000}%
\pgfsetstrokecolor{currentstroke}%
\pgfsetdash{}{0pt}%
\pgfsys@defobject{currentmarker}{\pgfqpoint{0.000000in}{-0.048611in}}{\pgfqpoint{0.000000in}{0.000000in}}{%
\pgfpathmoveto{\pgfqpoint{0.000000in}{0.000000in}}%
\pgfpathlineto{\pgfqpoint{0.000000in}{-0.048611in}}%
\pgfusepath{stroke,fill}%
}%
\begin{pgfscope}%
\pgfsys@transformshift{6.859700in}{0.421603in}%
\pgfsys@useobject{currentmarker}{}%
\end{pgfscope}%
\end{pgfscope}%
\begin{pgfscope}%
\definecolor{textcolor}{rgb}{0.000000,0.000000,0.000000}%
\pgfsetstrokecolor{textcolor}%
\pgfsetfillcolor{textcolor}%
\pgftext[x=6.859700in,y=0.324381in,,top]{\color{textcolor}\sffamily\fontsize{10.000000}{12.000000}\selectfont 0.5}%
\end{pgfscope}%
\begin{pgfscope}%
\pgfsetbuttcap%
\pgfsetroundjoin%
\definecolor{currentfill}{rgb}{0.000000,0.000000,0.000000}%
\pgfsetfillcolor{currentfill}%
\pgfsetlinewidth{0.803000pt}%
\definecolor{currentstroke}{rgb}{0.000000,0.000000,0.000000}%
\pgfsetstrokecolor{currentstroke}%
\pgfsetdash{}{0pt}%
\pgfsys@defobject{currentmarker}{\pgfqpoint{0.000000in}{-0.048611in}}{\pgfqpoint{0.000000in}{0.000000in}}{%
\pgfpathmoveto{\pgfqpoint{0.000000in}{0.000000in}}%
\pgfpathlineto{\pgfqpoint{0.000000in}{-0.048611in}}%
\pgfusepath{stroke,fill}%
}%
\begin{pgfscope}%
\pgfsys@transformshift{7.933336in}{0.421603in}%
\pgfsys@useobject{currentmarker}{}%
\end{pgfscope}%
\end{pgfscope}%
\begin{pgfscope}%
\definecolor{textcolor}{rgb}{0.000000,0.000000,0.000000}%
\pgfsetstrokecolor{textcolor}%
\pgfsetfillcolor{textcolor}%
\pgftext[x=7.933336in,y=0.324381in,,top]{\color{textcolor}\sffamily\fontsize{10.000000}{12.000000}\selectfont 1.}%
\end{pgfscope}%
\begin{pgfscope}%
\definecolor{textcolor}{rgb}{0.000000,0.000000,0.000000}%
\pgfsetstrokecolor{textcolor}%
\pgfsetfillcolor{textcolor}%
\pgftext[x=6.859700in,y=0.134413in,,top]{\color{textcolor}\sffamily\fontsize{10.000000}{12.000000}\selectfont similarity}%
\end{pgfscope}%
\begin{pgfscope}%
\pgfsetbuttcap%
\pgfsetroundjoin%
\definecolor{currentfill}{rgb}{0.000000,0.000000,0.000000}%
\pgfsetfillcolor{currentfill}%
\pgfsetlinewidth{0.803000pt}%
\definecolor{currentstroke}{rgb}{0.000000,0.000000,0.000000}%
\pgfsetstrokecolor{currentstroke}%
\pgfsetdash{}{0pt}%
\pgfsys@defobject{currentmarker}{\pgfqpoint{-0.048611in}{0.000000in}}{\pgfqpoint{0.000000in}{0.000000in}}{%
\pgfpathmoveto{\pgfqpoint{0.000000in}{0.000000in}}%
\pgfpathlineto{\pgfqpoint{-0.048611in}{0.000000in}}%
\pgfusepath{stroke,fill}%
}%
\begin{pgfscope}%
\pgfsys@transformshift{5.764592in}{2.590348in}%
\pgfsys@useobject{currentmarker}{}%
\end{pgfscope}%
\end{pgfscope}%
\begin{pgfscope}%
\pgfsetbuttcap%
\pgfsetroundjoin%
\definecolor{currentfill}{rgb}{0.000000,0.000000,0.000000}%
\pgfsetfillcolor{currentfill}%
\pgfsetlinewidth{0.803000pt}%
\definecolor{currentstroke}{rgb}{0.000000,0.000000,0.000000}%
\pgfsetstrokecolor{currentstroke}%
\pgfsetdash{}{0pt}%
\pgfsys@defobject{currentmarker}{\pgfqpoint{-0.048611in}{0.000000in}}{\pgfqpoint{0.000000in}{0.000000in}}{%
\pgfpathmoveto{\pgfqpoint{0.000000in}{0.000000in}}%
\pgfpathlineto{\pgfqpoint{-0.048611in}{0.000000in}}%
\pgfusepath{stroke,fill}%
}%
\begin{pgfscope}%
\pgfsys@transformshift{5.764592in}{1.516712in}%
\pgfsys@useobject{currentmarker}{}%
\end{pgfscope}%
\end{pgfscope}%
\begin{pgfscope}%
\pgfsetbuttcap%
\pgfsetroundjoin%
\definecolor{currentfill}{rgb}{0.000000,0.000000,0.000000}%
\pgfsetfillcolor{currentfill}%
\pgfsetlinewidth{0.803000pt}%
\definecolor{currentstroke}{rgb}{0.000000,0.000000,0.000000}%
\pgfsetstrokecolor{currentstroke}%
\pgfsetdash{}{0pt}%
\pgfsys@defobject{currentmarker}{\pgfqpoint{-0.048611in}{0.000000in}}{\pgfqpoint{0.000000in}{0.000000in}}{%
\pgfpathmoveto{\pgfqpoint{0.000000in}{0.000000in}}%
\pgfpathlineto{\pgfqpoint{-0.048611in}{0.000000in}}%
\pgfusepath{stroke,fill}%
}%
\begin{pgfscope}%
\pgfsys@transformshift{5.764592in}{0.443076in}%
\pgfsys@useobject{currentmarker}{}%
\end{pgfscope}%
\end{pgfscope}%
\begin{pgfscope}%
\pgfsetrectcap%
\pgfsetmiterjoin%
\pgfsetlinewidth{0.803000pt}%
\definecolor{currentstroke}{rgb}{0.000000,0.000000,0.000000}%
\pgfsetstrokecolor{currentstroke}%
\pgfsetdash{}{0pt}%
\pgfpathmoveto{\pgfqpoint{5.764592in}{0.421603in}}%
\pgfpathlineto{\pgfqpoint{5.764592in}{2.611821in}}%
\pgfusepath{stroke}%
\end{pgfscope}%
\begin{pgfscope}%
\pgfsetrectcap%
\pgfsetmiterjoin%
\pgfsetlinewidth{0.803000pt}%
\definecolor{currentstroke}{rgb}{0.000000,0.000000,0.000000}%
\pgfsetstrokecolor{currentstroke}%
\pgfsetdash{}{0pt}%
\pgfpathmoveto{\pgfqpoint{7.954809in}{0.421603in}}%
\pgfpathlineto{\pgfqpoint{7.954809in}{2.611821in}}%
\pgfusepath{stroke}%
\end{pgfscope}%
\begin{pgfscope}%
\pgfsetrectcap%
\pgfsetmiterjoin%
\pgfsetlinewidth{0.803000pt}%
\definecolor{currentstroke}{rgb}{0.000000,0.000000,0.000000}%
\pgfsetstrokecolor{currentstroke}%
\pgfsetdash{}{0pt}%
\pgfpathmoveto{\pgfqpoint{5.764592in}{0.421603in}}%
\pgfpathlineto{\pgfqpoint{7.954809in}{0.421603in}}%
\pgfusepath{stroke}%
\end{pgfscope}%
\begin{pgfscope}%
\pgfsetrectcap%
\pgfsetmiterjoin%
\pgfsetlinewidth{0.803000pt}%
\definecolor{currentstroke}{rgb}{0.000000,0.000000,0.000000}%
\pgfsetstrokecolor{currentstroke}%
\pgfsetdash{}{0pt}%
\pgfpathmoveto{\pgfqpoint{5.764592in}{2.611821in}}%
\pgfpathlineto{\pgfqpoint{7.954809in}{2.611821in}}%
\pgfusepath{stroke}%
\end{pgfscope}%
\begin{pgfscope}%
\definecolor{textcolor}{rgb}{0.000000,0.000000,0.000000}%
\pgfsetstrokecolor{textcolor}%
\pgfsetfillcolor{textcolor}%
\pgftext[x=6.859700in,y=2.695154in,,base]{\color{textcolor}\sffamily\fontsize{12.000000}{14.400000}\selectfont Mostly dissimilar}%
\end{pgfscope}%
\begin{pgfscope}%
\pgfsetbuttcap%
\pgfsetmiterjoin%
\definecolor{currentfill}{rgb}{1.000000,1.000000,1.000000}%
\pgfsetfillcolor{currentfill}%
\pgfsetlinewidth{0.000000pt}%
\definecolor{currentstroke}{rgb}{0.000000,0.000000,0.000000}%
\pgfsetstrokecolor{currentstroke}%
\pgfsetstrokeopacity{0.000000}%
\pgfsetdash{}{0pt}%
\pgfpathmoveto{\pgfqpoint{8.392853in}{0.421603in}}%
\pgfpathlineto{\pgfqpoint{10.583070in}{0.421603in}}%
\pgfpathlineto{\pgfqpoint{10.583070in}{2.611821in}}%
\pgfpathlineto{\pgfqpoint{8.392853in}{2.611821in}}%
\pgfpathclose%
\pgfusepath{fill}%
\end{pgfscope}%
\begin{pgfscope}%
\pgfpathrectangle{\pgfqpoint{8.392853in}{0.421603in}}{\pgfqpoint{2.190217in}{2.190217in}}%
\pgfusepath{clip}%
\pgfsys@transformshift{8.392853in}{0.421603in}%
\pgftext[left,bottom]{\pgfimage[interpolate=true,width=2.193333in,height=2.193333in]{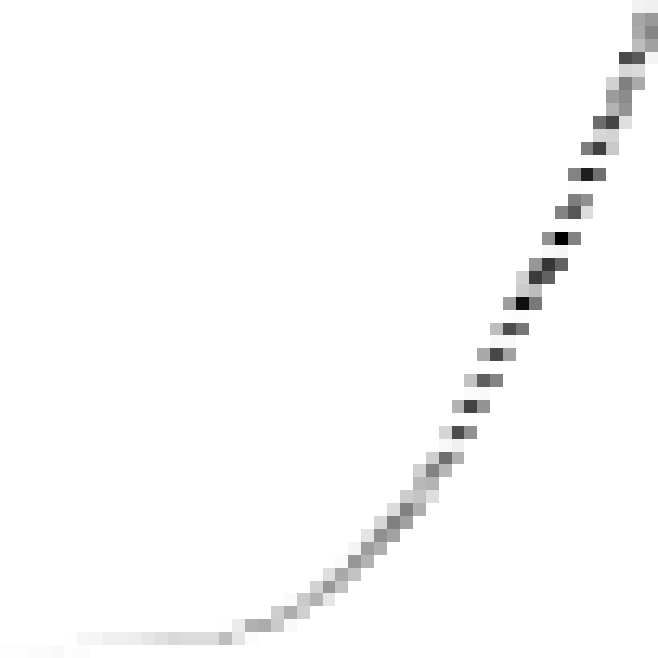}}%
\end{pgfscope}%
\begin{pgfscope}%
\pgfsetbuttcap%
\pgfsetroundjoin%
\definecolor{currentfill}{rgb}{0.000000,0.000000,0.000000}%
\pgfsetfillcolor{currentfill}%
\pgfsetlinewidth{0.803000pt}%
\definecolor{currentstroke}{rgb}{0.000000,0.000000,0.000000}%
\pgfsetstrokecolor{currentstroke}%
\pgfsetdash{}{0pt}%
\pgfsys@defobject{currentmarker}{\pgfqpoint{0.000000in}{-0.048611in}}{\pgfqpoint{0.000000in}{0.000000in}}{%
\pgfpathmoveto{\pgfqpoint{0.000000in}{0.000000in}}%
\pgfpathlineto{\pgfqpoint{0.000000in}{-0.048611in}}%
\pgfusepath{stroke,fill}%
}%
\begin{pgfscope}%
\pgfsys@transformshift{8.414325in}{0.421603in}%
\pgfsys@useobject{currentmarker}{}%
\end{pgfscope}%
\end{pgfscope}%
\begin{pgfscope}%
\definecolor{textcolor}{rgb}{0.000000,0.000000,0.000000}%
\pgfsetstrokecolor{textcolor}%
\pgfsetfillcolor{textcolor}%
\pgftext[x=8.414325in,y=0.324381in,,top]{\color{textcolor}\sffamily\fontsize{10.000000}{12.000000}\selectfont 0.}%
\end{pgfscope}%
\begin{pgfscope}%
\pgfsetbuttcap%
\pgfsetroundjoin%
\definecolor{currentfill}{rgb}{0.000000,0.000000,0.000000}%
\pgfsetfillcolor{currentfill}%
\pgfsetlinewidth{0.803000pt}%
\definecolor{currentstroke}{rgb}{0.000000,0.000000,0.000000}%
\pgfsetstrokecolor{currentstroke}%
\pgfsetdash{}{0pt}%
\pgfsys@defobject{currentmarker}{\pgfqpoint{0.000000in}{-0.048611in}}{\pgfqpoint{0.000000in}{0.000000in}}{%
\pgfpathmoveto{\pgfqpoint{0.000000in}{0.000000in}}%
\pgfpathlineto{\pgfqpoint{0.000000in}{-0.048611in}}%
\pgfusepath{stroke,fill}%
}%
\begin{pgfscope}%
\pgfsys@transformshift{9.487961in}{0.421603in}%
\pgfsys@useobject{currentmarker}{}%
\end{pgfscope}%
\end{pgfscope}%
\begin{pgfscope}%
\definecolor{textcolor}{rgb}{0.000000,0.000000,0.000000}%
\pgfsetstrokecolor{textcolor}%
\pgfsetfillcolor{textcolor}%
\pgftext[x=9.487961in,y=0.324381in,,top]{\color{textcolor}\sffamily\fontsize{10.000000}{12.000000}\selectfont 0.5}%
\end{pgfscope}%
\begin{pgfscope}%
\pgfsetbuttcap%
\pgfsetroundjoin%
\definecolor{currentfill}{rgb}{0.000000,0.000000,0.000000}%
\pgfsetfillcolor{currentfill}%
\pgfsetlinewidth{0.803000pt}%
\definecolor{currentstroke}{rgb}{0.000000,0.000000,0.000000}%
\pgfsetstrokecolor{currentstroke}%
\pgfsetdash{}{0pt}%
\pgfsys@defobject{currentmarker}{\pgfqpoint{0.000000in}{-0.048611in}}{\pgfqpoint{0.000000in}{0.000000in}}{%
\pgfpathmoveto{\pgfqpoint{0.000000in}{0.000000in}}%
\pgfpathlineto{\pgfqpoint{0.000000in}{-0.048611in}}%
\pgfusepath{stroke,fill}%
}%
\begin{pgfscope}%
\pgfsys@transformshift{10.561597in}{0.421603in}%
\pgfsys@useobject{currentmarker}{}%
\end{pgfscope}%
\end{pgfscope}%
\begin{pgfscope}%
\definecolor{textcolor}{rgb}{0.000000,0.000000,0.000000}%
\pgfsetstrokecolor{textcolor}%
\pgfsetfillcolor{textcolor}%
\pgftext[x=10.561597in,y=0.324381in,,top]{\color{textcolor}\sffamily\fontsize{10.000000}{12.000000}\selectfont 1.}%
\end{pgfscope}%
\begin{pgfscope}%
\definecolor{textcolor}{rgb}{0.000000,0.000000,0.000000}%
\pgfsetstrokecolor{textcolor}%
\pgfsetfillcolor{textcolor}%
\pgftext[x=9.487961in,y=0.134413in,,top]{\color{textcolor}\sffamily\fontsize{10.000000}{12.000000}\selectfont similarity}%
\end{pgfscope}%
\begin{pgfscope}%
\pgfsetbuttcap%
\pgfsetroundjoin%
\definecolor{currentfill}{rgb}{0.000000,0.000000,0.000000}%
\pgfsetfillcolor{currentfill}%
\pgfsetlinewidth{0.803000pt}%
\definecolor{currentstroke}{rgb}{0.000000,0.000000,0.000000}%
\pgfsetstrokecolor{currentstroke}%
\pgfsetdash{}{0pt}%
\pgfsys@defobject{currentmarker}{\pgfqpoint{-0.048611in}{0.000000in}}{\pgfqpoint{0.000000in}{0.000000in}}{%
\pgfpathmoveto{\pgfqpoint{0.000000in}{0.000000in}}%
\pgfpathlineto{\pgfqpoint{-0.048611in}{0.000000in}}%
\pgfusepath{stroke,fill}%
}%
\begin{pgfscope}%
\pgfsys@transformshift{8.392853in}{2.590348in}%
\pgfsys@useobject{currentmarker}{}%
\end{pgfscope}%
\end{pgfscope}%
\begin{pgfscope}%
\pgfsetbuttcap%
\pgfsetroundjoin%
\definecolor{currentfill}{rgb}{0.000000,0.000000,0.000000}%
\pgfsetfillcolor{currentfill}%
\pgfsetlinewidth{0.803000pt}%
\definecolor{currentstroke}{rgb}{0.000000,0.000000,0.000000}%
\pgfsetstrokecolor{currentstroke}%
\pgfsetdash{}{0pt}%
\pgfsys@defobject{currentmarker}{\pgfqpoint{-0.048611in}{0.000000in}}{\pgfqpoint{0.000000in}{0.000000in}}{%
\pgfpathmoveto{\pgfqpoint{0.000000in}{0.000000in}}%
\pgfpathlineto{\pgfqpoint{-0.048611in}{0.000000in}}%
\pgfusepath{stroke,fill}%
}%
\begin{pgfscope}%
\pgfsys@transformshift{8.392853in}{1.516712in}%
\pgfsys@useobject{currentmarker}{}%
\end{pgfscope}%
\end{pgfscope}%
\begin{pgfscope}%
\pgfsetbuttcap%
\pgfsetroundjoin%
\definecolor{currentfill}{rgb}{0.000000,0.000000,0.000000}%
\pgfsetfillcolor{currentfill}%
\pgfsetlinewidth{0.803000pt}%
\definecolor{currentstroke}{rgb}{0.000000,0.000000,0.000000}%
\pgfsetstrokecolor{currentstroke}%
\pgfsetdash{}{0pt}%
\pgfsys@defobject{currentmarker}{\pgfqpoint{-0.048611in}{0.000000in}}{\pgfqpoint{0.000000in}{0.000000in}}{%
\pgfpathmoveto{\pgfqpoint{0.000000in}{0.000000in}}%
\pgfpathlineto{\pgfqpoint{-0.048611in}{0.000000in}}%
\pgfusepath{stroke,fill}%
}%
\begin{pgfscope}%
\pgfsys@transformshift{8.392853in}{0.443076in}%
\pgfsys@useobject{currentmarker}{}%
\end{pgfscope}%
\end{pgfscope}%
\begin{pgfscope}%
\pgfsetrectcap%
\pgfsetmiterjoin%
\pgfsetlinewidth{0.803000pt}%
\definecolor{currentstroke}{rgb}{0.000000,0.000000,0.000000}%
\pgfsetstrokecolor{currentstroke}%
\pgfsetdash{}{0pt}%
\pgfpathmoveto{\pgfqpoint{8.392853in}{0.421603in}}%
\pgfpathlineto{\pgfqpoint{8.392853in}{2.611821in}}%
\pgfusepath{stroke}%
\end{pgfscope}%
\begin{pgfscope}%
\pgfsetrectcap%
\pgfsetmiterjoin%
\pgfsetlinewidth{0.803000pt}%
\definecolor{currentstroke}{rgb}{0.000000,0.000000,0.000000}%
\pgfsetstrokecolor{currentstroke}%
\pgfsetdash{}{0pt}%
\pgfpathmoveto{\pgfqpoint{10.583070in}{0.421603in}}%
\pgfpathlineto{\pgfqpoint{10.583070in}{2.611821in}}%
\pgfusepath{stroke}%
\end{pgfscope}%
\begin{pgfscope}%
\pgfsetrectcap%
\pgfsetmiterjoin%
\pgfsetlinewidth{0.803000pt}%
\definecolor{currentstroke}{rgb}{0.000000,0.000000,0.000000}%
\pgfsetstrokecolor{currentstroke}%
\pgfsetdash{}{0pt}%
\pgfpathmoveto{\pgfqpoint{8.392853in}{0.421603in}}%
\pgfpathlineto{\pgfqpoint{10.583070in}{0.421603in}}%
\pgfusepath{stroke}%
\end{pgfscope}%
\begin{pgfscope}%
\pgfsetrectcap%
\pgfsetmiterjoin%
\pgfsetlinewidth{0.803000pt}%
\definecolor{currentstroke}{rgb}{0.000000,0.000000,0.000000}%
\pgfsetstrokecolor{currentstroke}%
\pgfsetdash{}{0pt}%
\pgfpathmoveto{\pgfqpoint{8.392853in}{2.611821in}}%
\pgfpathlineto{\pgfqpoint{10.583070in}{2.611821in}}%
\pgfusepath{stroke}%
\end{pgfscope}%
\begin{pgfscope}%
\definecolor{textcolor}{rgb}{0.000000,0.000000,0.000000}%
\pgfsetstrokecolor{textcolor}%
\pgfsetfillcolor{textcolor}%
\pgftext[x=9.487961in,y=2.695154in,,base]{\color{textcolor}\sffamily\fontsize{12.000000}{14.400000}\selectfont Mostly similar}%
\end{pgfscope}%
\end{pgfpicture}%
\makeatother%
\endgroup%